\documentclass{article}

\usepackage[preprint]{neurips_2023}

\usepackage{graphicx}
\usepackage{caption}
\usepackage{booktabs} 
\usepackage{mwe}
\usepackage{hyperref}
\usepackage{xspace} 

\usepackage[utf8]{inputenc} 
\usepackage[T1]{fontenc}    
\usepackage{url}            
\usepackage{booktabs}       
\usepackage{amsfonts}       
\usepackage{nicefrac}       
\usepackage{microtype}      
\usepackage{xcolor}         
\usepackage{amsmath}
\usepackage{amssymb}
\usepackage{mathtools}
\usepackage{amsthm}
\usepackage{physics}

\theoremstyle{plain}
\newtheorem{theorem}{Theorem}[section]

\newtheorem{lemma}[theorem]{Lemma}

\theoremstyle{definition}

\newtheorem{assumption}[theorem]{Assumption}
\theoremstyle{remark}

\newcommand{\N}{\mathbb{N}}
\newcommand{\E}{\mathbb{E}}

\newcommand{\KRR}{\ensuremath{\mathsf{KRR}}\xspace}

\newcommand{\rmse}{\textsc{rmse}\xspace}

\newcommand{\krr}{\textsc{krr}\xspace}

\newcommand{\R}{\ensuremath{\mathbb{R}}}

\newcommand{\eps}{\ensuremath{\varepsilon}}
\newcommand{\ddim}{\ensuremath{\mathsf{dd}}}

\newcommand{\etal}{\emph{et{.}al{.}}}
\newcommand{\jeff}[1]{{\color{blue}$\langle$\textsf{\small Jeff: } #1$\rangle$}}

\title{Locally Adaptive and Differentiable Regression \thanks{
Thanks to NSF CDS\&E-1953350, CCF-2115677, IIS-1816149, and IIS-2311954.}}

\author{
  Mingxuan Han \\
  School of Computing\\
  The University of Utah\\
  Salt Lake City, UT 84112 \\
  \texttt{u1209601@utah.edu} \\
  \And
  Varun Shankar \\
  School of Computing \\
  The University of Utah \\
  Salt Lake City, UT 84112 \\
  \texttt{shankar@cs.utah.edu} \\
  \AND
  Chenglong Ye \\
  Department of Statistics \\
  The University of Kentucky  \\
  \texttt{chenglong.ye@uky.edu} \\
  \And
  Jeff M Phillips   \\
  School of Computing \\
  The University of Utah \\
  Salt Lake City, UT 84112 \\
  \texttt{jeffp@cs.utah.edu} \\
}

\begin{document}
\maketitle

\begin{abstract}
Over-parameterized models like deep nets and random forests have become very popular in machine learning. However, the natural goals of continuity and differentiability, common in regression models, are now often ignored in modern overparametrized, locally-adaptive models. We propose a general framework to construct a global continuous and differentiable model based on a weighted average of locally learned models in corresponding local regions. This model is competitive in dealing with data with different densities or scales of function values in different local regions. We demonstrate that when we mix kernel ridge and polynomial regression terms in the local models, and stitch them together continuously, we achieve faster statistical convergence in theory and improved performance in various practical settings.  

\end{abstract}

\section{Introduction}

Regression is one of the fundamental tasks within machine learning.  Historically, most techniques used polynomial or kernel models, and these had the advantage that they generalize well on noisy data and are continuous and even differentiable.  However, recently with the rise of enormous, accurate data sets and abundant computational power, over-parameterized models like deep nets and random forests have become very popular.  Unfortunately, these complex and often opaque models either eschew guarantees on continuity and differentiability to achieve highly accurate, and/or are not really locally adaptive model as one may presume.  

In this paper, we study regression models that achieve both (a) \textbf{continuity} in prediction and (b) \textbf{local adaptivity} in the model.  
By \emph{continuity}, we mean that as the point of prediction changes infinitesimally, the output of that prediction also changes infinitesimally.  At the minimum we seek $C^0$-continuity so the value predicted is stable, but our methods will address $C^1$-continuity (and in general $C^t$-continuity) that can ensure properties like continuity in gradients as well.  Methods like decision trees or neural networks with ReLu activation do not achieve $C^0$ or $C^1$ continuity, respectively; at decision thresholds there are jumps in prediction.  
By \emph{locally adaptivity}, we mean that a prediction at a query point only depends on nearby data points (for some notion of nearby).  In particular, the prediction for a query for a locally adaptive model should not be affected by the modification of training data far away (e.g., an outlier, or an unrelated domain adaptation).  Such locally adaptive models are important when the scale of response values change significantly outside of local regions.  
Discontinuous models like decision trees, random forests, or ReLu-activated neural networks might achieve this, but models with global blending functions or continuous sigmoid-activated neural networks do not.  
In general, a difficulty arises in that locally-defined models typically have seams between regions which cause unintuitive jumps between close predicted values, whereas continuous models typically do not actually ensure locality.  

In this paper we describe a highly effective and efficient regression model that is both locally adaptive and $C^t$-continuous for any constant $t$.  
To achieve this, we adapt techniques from interpolation theory to build highly adaptive and overparametrized models -- but without relying on the necessity to interpolate the data at high precision.  Our method starts by identifying overlapping local neighborhoods of a dataset, and building a regression model for each region.  These local models allow the method to adapt to local variation in data scale and density.  
The key step is to then ``stitch'' these local models together in a continuous way that forms a global continuous model, maintaining local adaptivity.  

Our approach towards stitching local models is based on a ``partition of unity" (PU) perspective where each local region is defined by a Euclidean ball, and a weighting function from a centrally symmetric radial-basis kernel.  Then given a query point, the local regions which contain this query points are averaged together proportional to their kernel weight.  In the interior of these regions, as long as the models are continuous and the kernels are continuous, then the global model inherits that continuity.  However, the boundary of a model region presents a challenge.  To address this we employ Wendland kernels which are compactly-supported, continuous, and differentiable reproducing radial-basis kernels (over a finite dimensional space) defined over a fixed radius.  Critically, the function values go to zero and derivatives vanish at the model boundary.  While the implications of these remarkable kernels are well-known within the field of interpolation theory and kriging, they seem mostly unknown within AI and machine learning.  

Given this overall approach to continuously stitch together local models, we can use any continuous and differentiable regression approach within each local region.  We find that a method which uses a combination of kernel and polynomial basis elements, which we call KRR-POLY, does the best job of fitting local data and generalizing to held out data.  Moreover, we formalize how this method provides improved statistical convergence in comparison to the simpler kernel ridge regression model.  
In fact, we observe that our overall method PU-KRR-POLY outperforms in generalization in comparison to other powerful regression models on a variety of data sets, especially when the data requires local adaptation.  

Moreover, because our method is differentiable, we can also directly compute the partial derivatives of our models.  This is a critical step in many regression tasks such as those for scientific simulation, and as a result our method does not need to rely on discrete differentiation schemes, and does not suffer from spikes or discontinuities which can cause anomalous behavior.  Due to this improved statistical convergence, guaranteed continuity, and direct calculation of derivatives, our method is demonstrated to significantly improve in accuracy of derivatives with more training data while others models do not.

\section{Preliminaries}

We consider as input a set of $n$ data points $X \subset \R^d$, and for each $x_i \in X$ a measured response value $y_i \in \R$.  We assume that $y_i = f(x_i) + \eps_i$ where $\eps_i$ is a small independent noise term and $f$ is a continuous function $f$.  In particular, we focus on data $X$ that may not be uniformly distributed over some domain, and functions $f$ which may have different properties (e.g., in terms of volatility or closeness to $0$) in different local regions populated by $X$.  And moreover, the variance of the error distribution governing each $\eps_i$ might be smaller where the function value $f(x_i)$ is closer to $0$.

We use the term \emph{continuous} to generically capture different degrees of continuity of a function $f$.  A $C^0(\R^d)$ continuous function ensures small changes in the argument in $\R^d$ result in proportionally small changes in $f$.   
However, functions can be \emph{$C^t(\R^d)$-continuous} which also ensures that the higher-order partial derivatives, up to the $t$th partial derivative in any direction is well-defined.  Our discussion and algorithms are applicable for $C^t$-continuous functions for any finite $t$, but our experimental evaluations will mostly focus on $C^1$-continuous functions where we can take derivatives.  

\paragraph*{Kernels}
We will use kernel methods on Euclidean data, and this relies on a radial kernel $K : \R^d \times \R^d \to \R$.  Most such kernel methods use the Gaussian $K(x,p) = \exp(-\|x-p\|^2/\sigma^2)$ or Laplace $K(x,p) = \exp(-\|x-p\| / \sigma)$ kernels; we employ Gaussians for our local models which are $C_t$-continuous everywhere for any value $t$.  

\paragraph*{Reproducing Kernels}
Since the idea of \KRR depends on reproducing kernel Hilbert space (RKHS), we give a brief summary here. Any positive semidefinite kernel function $K: \mathcal{X} \times \mathcal{X} \to \R$ can define a RKHS. For any $f$ in the Hilbert Space $\cal{H}$, we can represent it by its inner product $\langle \cdot , \cdot \rangle_{\cal{H}}$ in $\cal{H}$ with $K(x, \cdot)$, such that $\langle f, K(x, \cdot) \rangle_{\cal{H}} = f(x)$, $f \in \cal{H}$. We denote the norm of $f$ in $\cal{H}$, $\|f\|_{\cal{H}}$ based on inner product $\langle \cdot , \cdot \rangle_{\cal{H}}$ as $\sqrt{\langle f,f \rangle_{\cal{H}}}$.  We define $\|f\|_{2}$ as $(\int_{X} f^{2}(x) \mathrm{d} \mathbb{P}(x))^{1/2}$, where $\mathbb{P}$ is a (often implicitly uniform) distribution of $x$.  
Expectation $\E$ is taken over all $(x_i,y_i)$ pairs from $(X,Y)$ which are assumed drawn iid from $\mathbb{P}$, since the estimator $\hat f$ is trained from the data ${(x_i,y_i)}_{i=1}^{n}$. For example, if $\hat{f}(x) = \beta \cdot x$ then $\E\| \hat{f}(x) \|_{2}^{2} = \E(\beta^2 x^2 \mathrm{d} \mathbb{P}(x))$. 

\paragraph*{Wendland Kernels}
For the weighting of local regions, we use Wendland kernels~\cite{Dehnen_2012}. The Wendland kernels $\phi_{d,t}(r)$ constitute a two-parameter family of compactly-supported and (strictly) positive-definite radial kernels belonging to $C^{2t}(\R^d)$, and have widely been used for interpolation~\cite{Wendland:2004}. Note that these are only positive-definite up to and including $\R^d$, where $d$ is a fixed finite dimension; in contrast, more common kernels in machine learning like Gaussians are positive-definite for all dimensions, but cannot be compactly-supported. In fact, it is impossible to generate a radial kernel that is both positive-definite and compactly-supported for all dimensions~\cite{wu1995}.

A variety of Wendland kernels exists for any finite values $d,t > 0$; see~\cite{Wendland:2004,Fasshauer:2007} for examples. In this work, we restrict ourselves to the $C^2$ Wendland kernel in $\R^d$ given by $\phi_{d,1}(r)$, since our experimental evaluations mostly focus on $C^1$ functions. For instance, when $d=3$, this kernel, which is radially symmetric, is given by
\begin{align*} 
\phi_{3,1}(v) = (1-v/r)^4_{+} (1 + 4 v/r),
\end{align*}
where $v = \|x-p\|$ for $x,p \in \R^d$, where $r$ is the radius of a local region, and the term $(1- v/r)^4_{+} = (1- v/r)^4$ iff $(1- v/r)>0$, and is $0$ otherwise. As a natural consequence of this choice of $r$, when regions overlap, multiple Wendland kernels are non-zero over the overlapping volumes.

\paragraph*{Locally Adaptive Regression}
Parametric regression models, e.g., polynomials, treat all parts of the domain and all parts of the data equally. Local regressions, which learn testing patterns only based on its vicinity in training samples, can be dated back to \cite{Vapnik_1991, vapnik_local}. Such ideas in local learning have two main advantages over global methods. First, local learning can be computational efficient when dealing with large scale data \cite{JMLR:v16:zhang15d}. Second, local learning can easily adjust to the properties of training data in each sub regions of the input space \cite{vapnik_local}. Relevant algorithms and theoretical analysis in local regressions have been proposed in, e.g.: \cite{atkeson1997locally, loader2006local, JMLR:v16:zhang15d, Benefitofinterpolation}. Kernel methods (e.g., kernel ridge regression or Nadaraya-Watson kernel regression) typically enforce a fixed bandwidth globally, so while they can somewhat adapt to the local regions, enforce a global notion of scale.  
There exists variants of kernel methods which (mostly for scalability concerns) build local kernel regression models and then combine these together (while not guaranteeing continuity), and include local-svm~\cite{Ingo_Steinwart} and knn-svm~\cite{Hable_Robert}.  We will compare against these.  

Other approaches like decision trees are explicit in finding local regions for which simple models can be fit, and this is inherited in random forests -- although these do not attempt to guarantee continuity.  Neural networks can also represent local regions with different model properties -- although implicitly.  Using continuous and differentiable activation functions guarantees continuity, but typically the chaining of such functions leads to very high derivatives in regions if not carefully controlled.

\section{The PU-Stitched Regression Model}
\label{sec:algorithm}

We describe next how to build and then evaluate the newly proposed PU-Stitched Regression model.  
\subsection{Building Models on Local Regions}
Building our model involves two stages: (1) identifying the local regions, and (2) building a model on each local region.  

We model each local region as a Euclidean ball $B_j \subset \R^d$ with center $c_j$.  We choose the centers as a subset of the data points in a method that allows it to adapt to the data.  Given a center point $c_j \in X$, we set the radius of $B_j$ so that it contains $h$ points; we set $h = 100$ as default in our experiments. While $h=100$ worked consistently well in our experiments, see for instance the ablation study in Section \ref{sec:2DTests}, a user may need to tune this based on data distribution and need to local adaptivity.
We keep track of all points which are in no local regions, iteratively choose new points $c_{j+1} \in X$ (arbitrarily) among those in no regions, and create a new region $B_{j+1}$ around it covering more points, until all points are covered. We then also add one large region (infinite ball $B_0$) that contains any query.  

Next we build a local regression model $\hat{f}_j : B_j \to \R$ on the data in each region  $X_j = X \cap B_j$.  
As a baseline model for $\hat{f}_j$ we consider building a kernel ridge regression model using a Gaussian kernel.  That is let $K$ be the $h \times h$ matrix where entry $K_{s,t} = K(x_s,x_t)$ is the kernel similarity between a pair of points.  We build a model $\hat{f}_j(q) = \sum_{x_i \in X_j} \alpha_i K(x_i,q)$ as $\alpha = (K + \eta I)^{-1} y$ where $y$ is the vector of response terms from $X_j$ and $\eta > 0$ is a small ridge parameter. This optimizes the expression
\begin{align} \label{krr_eq}
    \min_{\alpha \in \R^n} \;\;\; (y - K\alpha)^{T}(y - K\alpha) + \eta \alpha^T K \alpha.  
\end{align}
By default, we can set the bandwidth $\sigma$ of the Gaussian kernel as the mean of all pairwise distances in $X_j$, and set the ridge term $\eta$ (at about $0.01\%$ of the average response value), since on a small local patch we should be able to fit data well, and this term serves mainly to well-condition the $K+\eta I$ matrix.  While there are many theoretical methods in the literature to determine bandwidth and ridge parameters ~\cite{Cucker_bestchoices, Caponnetto_optimalrates, NIPS2011_51ef186e}, it is common in practice to use cross-validation~\cite{Wahba1978/79, 10.1214/12-AOS1063, Gu_MA, ZHANG201595}.  In our experiments we do a small $5 \times 5$ grid search over $\eta$ and $\sigma$ on a held-out set.  

However, any local model could be built on a local region's data.  We find that a model that combines kernel and polynomial terms performs exceptionally well, and discuss it in more detail in Section \ref{sec:krr+poly}

\subsection{Partition of Unity Combination of Local Models}\label{sec:derivative}
To evaluate the PU-regression model at a query point $q \in \R^d$, we first need to determine all of the regions $B_j$ which contain $q$.  Let $J(q) = \{j_1, j_2, \ldots\}$ be the set of indices of regions which contain $q$.  
For each region $j \in J(q)$, we evaluate the Wendland kernel at $q$ and obtain a weight $w_j$ $= \phi_{d,t}(\|q-c_j\|)$.  
Recall that each query falls in the $B_0$ region, and in this region we set $w_0 = 1\text{e-}5$ as a small constant weight.  

We then employ the partition of unity (PU) approach to normalize these weights.  Set $W(q) = \sum_{j \in J(q)} w_j$ and then $w'_j = w_j/W(q)$; thus each $w'_j > 0$ and $\sum_{j \in J(q)}w'_j = 1$.  Since $w_0 > 0$, then $W(q) >0$ and we do not divide by 0.  
Now to evaluate the global regression function $\hat f$ at a point $q$ we compute the PU-weighted average of function values from each region $q$ falls in as  
\[
\hat f(q) = \sum_{j \in J(q)} w'_j \hat f_j(q).
\]


\paragraph*{Continuous, Differentiable}
To guarantee $C^t$-continuity of the regression function $\hat f$ we can leverage the PU-framework, as long as the local functions $\hat f_j$ and the choice of Wendland kernels also satisfy that property.  
Building on the PU-framework using Wendland kernels~\cite{Dehnen_2012}, to achieve global continuous and differentiable we only need to make sure the local models are continuous and differentiable. 

As a simple example of the usefulness of this, we can directly calculate all partial derivatives and the gradient of the modeled function $\hat f$.  
This works as long as the local model $\hat f_j(q)$ and the normalized Wendland kernel weight $w'(\|c_j-q\|)$ are at least $C^1(\R^d)$ with respect to $q \in \R^d$ (recalling that $c_j$ is the center of the model $\hat f_j$ and kernel weight function $w_j$). For any coordinate $q_i$ of $q$ the partial derivative of $\hat f$ is

\begin{align}\label{eq:derivatives}
\hat f'_i(q) = \frac{\mathrm{d}}{\mathrm{d} q_i} \hat f(q) = \sum_{j \in J_q} \frac{\mathrm{d} w'_j}{\mathrm{d} q_i} \hat f_j(q) + \frac{\mathrm{d} \hat f_j}{\mathrm{d} q_i} w'_j(\|c_j - q\|).  
\end{align}

Due to both the compact-support and smoothness of the $C^2(\R^d)$ Wendland kernels, when $q$ is on the boundary of its support, $w_j(q)=0$ \emph{and} all partial derivatives are $0$ at $q$ also. Consequently, the stitching does not introduce boundary effects at the support boundaries, and the gradient of $\hat f$ at $q$ is simply $[\hat f'_1, \ldots, \hat f'_d]^\top$.  On the other hand, if one were to instead use a truncated Gaussian, the derivatives measured from different support regions at $q$ would not match, since the truncated Gaussian is neither smooth nor exactly zero at its boundary.


Any local model which satisfies continuity and differentiability can be combined to a global model using the PU-method we proposed. In this paper, we investigate KRR and KRR-POLY, which is explained in the next section.

\subsection{Related work on Locally Adaptive and Continuous Regression}
The most common weighting strategy of smaller models is to simply take the uniform average of all predictions~\cite{JMLR:v16:zhang15d,pmlr-v80-xu18f}, but in these works those are not localized. 
Other weighting strategies depend on the nearest neighbor or decisions trees which are not $C^0$ continuous between local models.

While there are many approaches that give rise to either continuous or locally adaptive regression models, we know of only one that achieves both.  This is a recently introduced method dubbed localKRR~\cite{hanlocal}.  Similar to our work, it starts by building a set of local models; for each it uses a KRR model, but again could use any continuous regression model.  Then on a query $q$ it identifies the $h$ nearest points at the center of each model (for $h > d+1$), and combines them together in a weighted average.  The weighted average is devised as a function of the distance between $q$ and the center point of each region, and ensures $C^0$-continuity.  

Compared to our results, this localKRR has a few short-comings.  First, it requires each model formed at query $q$ is devised based on the $h$ closest models, so effectively requires the space to be covered everywhere by $h > d+1$ models, not by only 1 model as in ours.  Also each local model does not have a self-define region for which it is used, it depends on the distribution of other models; in particular it is determined by the $h$th order Voronoi diagram of model centers.  For our method, as models are formed they determine a fixed ball subset of $\R^d$ for which it is used; each and every query $q$ in that ball invokes that model.  
Second, localKRR can only guarantee $C^0$-continuity, where as our approach can attain $C^t$-continuity for any constant $t$.  Thus for localKRR, one cannot everywhere compute gradients on the learned regression function.  
Finally, as we will see, our approach PU-KRR-POLY outperforms localKRR empirically in each experiment.  

Another approach by Belkin \etal~
\cite{belkin2018overfitting} considers a variant of Nadaraya-Watson kernel regression (NWKR) that nearly achieves these locally adaptive and continuous properties.  It uses a singular kernel with NWKR which interpolates points (as localKRR nearly does), and achieves $C^0$ continuity, but in practices achieves local adaptivity by truncating the kernels, which actually destroys the continuity at these truncation thresholds.  


\section{KRR-POLY}
\label{sec:krr+poly}
We borrow the idea of RBF interpolation augmented with polynomials (see~\cite{Fasshauer:2007} for an example) to propose a new variant of kernel ridge regression: KRR-POLY. We also provide new statistical analysis of this model under the noisy (non-interpolation) setting. Consider again eq(\ref{krr_eq}). Instead of estimating $f$ solely with a kernel expansion, we now also augment the kernel expansion with polynomials of total degree $\ell$ in $\R^d$. Letting $p_i:\R^d \to \R, i=1,\ldots, {\ell + d \choose d}$ be a basis for this space of polynomials, we now build a model $\hat{f}_j(q) = \sum_{x_i \in X_j} \alpha_i K(x_i,q) + \sum\limits_{i=1}^{\ell + d \choose d} \lambda_i p_i(q)$. To find the coefficients $\alpha_i$ and $\lambda_i$, we solve the following minimization problem:
\begin{align} 
    &\min_{\alpha \in \R^n}  \;\;\; (y - K\alpha - P\lambda)^{T}(y - K\alpha - P\lambda) + \eta \alpha^T K \alpha, \nonumber\\
    &\textit{s.t.} \; P^T \alpha =0 \label{poly-krr}
\end{align}
where $P_{ij} = p_j(x_i)$, $\lambda$ is the unknown vector of polynomial coefficients, but also a Lagrange multiplier that enforces the constraint $P^T \alpha = 0$. In this work, we choose $p_j$ to be the set of $d$-variate monomials up to degree $\ell$; we find $\ell=2$ is sufficient to induce significant advantage over regular KRR models. This additional constraint forces the kernel expansion $\sum_{x_i \in X_j} \alpha_i K(x_i,q)$ to be orthogonal to the polynomial terms, thereby ensuring that the overall approximant reproduces polynomials up to degree $\ell$~\cite{BAYONA20192337}. In addition, this constraint regularizes the far field of the RBF expansion~\cite{fornberg2002}. The above constraints can be collected into the following block linear system:
\begin{align*} 
    \begin{bmatrix}
    K + \eta I_{n} & P \\
    P^T & \mathbf{0}
    \end{bmatrix} \begin{bmatrix}
    \alpha \\
    \lambda 
    \end{bmatrix} = \begin{bmatrix}
    y \\ 
    \mathbf{0}
    \end{bmatrix}.
\end{align*}
This linear system has a unique solution provided the data locations are distinct, and if $P$ is of full-rank (see~\cite{Fasshauer:2007} for proof). However, if the data locations lie on a locally algebraic submanifold of $\R^d$, $P$ is likely to be rank-deficient. Thus, to ensure the generality of our technique, we solve the above linear system using the singular value decomposition (SVD) with thresholding of the singular values; we set the threshold at $10^{-10}$. This is equivalent to enforcing the constraint $P^T \alpha = 0$ in a least-squares sense. In the deterministic interpolation setting, the above approach leads to a convergence rate controlled by the polynomial terms~\cite{DavydovSchabackMinimal}.  Below, we discuss how to generalize these results to the statistical setting. It is important to note in the discussion above that the overall approximant stays the same regardless of the choice of the polynomial basis. In our proof below, for instance, we find it more convenient to use an orthonormal polynomial basis.

\paragraph*{Statistical Convergence}
We next show a statistical convergence rate for this new KRR-POLY model.  We do so in a similar form as \cite{JMLR:v16:zhang15d}'s bound for kernel regression.  For KRR, one can show that $\E \|\hat f - f\|_2^2 = O((\eta + \frac{1}{\eta n}) \cdot \|f\|_{\cal{H}}^2)$. 
In contrast we can show that this can be improved for KRR-POLY to 
$\E \|\hat f - f\|_2^2 = O((\eta + \frac{1}{\eta n}) n^{-(\ell+1)/d} \cdot \|f\|_{\cal{H}}^2)$, under the mild assumptions listed below.  

Define $g(x) = f(x) -  \sum_{k = 1}^{s} p_k(x) {\lambda}^{*}_{k}$, the projection of the generating function $f(x)$ on to the space orthogonal to the polynomial basis ${p_1(x),...,p_s(x)}$.
Given this fixed polynomial basis, the statistical convergence then only needs to learn the residual (captured in $g$) via kernel ridge regression. As \cite{BAYONA20192337} implies the $\|g\|_{\cal{H}}^{2} \le n^{-(\ell+1)/d}\|f\|_{\cal{H}}^{2}$, so our task is reduced.  
We can then apply existing bounds, and in this paper we choose to employ that of \cite{JMLR:v16:zhang15d}.  

Combining these insights together yields the following simplified form of our statistical convergence rate, after stating two assumptions.  

\begin{assumption}\label{conv_assump_1}
The unknown data generating function $f \in \cal{H}$, where $Y = f(X) + \eps$, for all $x_i \in \mathcal{X}$, we have $\E[(y_i - f(x_i))^2 \mid x_i ] \le \delta^2$ for some $\delta > 0$.  
\end{assumption}
\begin{assumption}\label{conv_assump_2}
For some $k \ge 2$, there is a constant $\rho < \infty$, s.t. $\E p_j^{2k}(X) \le \rho^{2k}$ for all $j \in \N$, where $\rho$ is a uniform upper bound for the moment.
\end{assumption}

This condition, in Assumption \ref{conv_assump_2}, regulates the tail behavior of the polynomial basis in the RKHS. The number of moments, $k$, depends on the choice of kernel $K$; for Gaussians it holds for $k=2$.

\begin{theorem}\label{conv_theorem}
Under Assumption \ref{conv_assump_1} and \ref{conv_assump_2}, for $f \in \cal{H}$, estimator $\hat{f}$ of KRR-POLY on $n$ training data points in $\R^d$, ridge parameter $\eta$, and polynomials of degree $\ell$ has mean square error: 
\begin{align*}
    \E \| \hat{f} - f \|_{2}^{2} = O((\eta +r(\eta)/n)n^{-(\ell+1)/d}\cdot \| f \|_{\cal{H}}^{2}). 
\end{align*}  
\end{theorem}

Theorem \ref{conv_theorem} is a consequence of the below lemma with a more detailed upper bound, and the simplifications described in the remark that follows.   
However, that more precise bound is nuanced and technical, and we need some additional notation.  Let $1/u_j=\|p_j(x)\|_{H}^{2}$ be the norm of each polynomial basis. Let $r(\eta) = \sum_{j = 1}^{\infty} \frac{1}{1 + \frac{\eta}{u_j}}$ be the effective dimensionality ~\cite{effective_dimension_tong} of the kernel. Let $u_{\infty} = \sum_{j = 1}^{\infty} u_j$ be the kernel trace, which is assumed to be finite and provides a rough estimate of the size of the kernel. Then $\beta_{d} = \sum_{j = d + 1}^{\infty} u_j$ describes the decay of the tail of the eigenvalues of $K$. The quantity $b(n,s,k) = \max \{ \sqrt{\max(k, \log s)}, \frac{\max(k,\log s)}{n^{\frac{1}{2} - \frac{1}{k}}}\}$ is a function of the number of moments $k$, where $s$ is the number of the of polynomial basis $\{p_i(x)\}$ for the multivariate polynomial space $\Pi_{\ell}^{d}$.

\begin{lemma}\label{conv_lemma}
Under assumption \ref{conv_assump_1} and \ref{conv_assump_2}, for $f \in \cal{H}$, estimator $\hat{f}$ of KRR-POLY has mean square error bound as: 
\begin{align*}
    \E \| \hat{f} - f \|_{2}^{2} & \le 12 \eta \| g\|_{\cal{H}}^{2} + \frac{12 \delta^{2}}{n} r(\eta)+s\cdot E\|p_k\|_2^2\cdot o(h^{2(l+1)}) 
    \\ & + \left(\frac{2 \delta^2}{\eta} + 4 \| g \|_{\cal{H}}^{2}\right) \cdot \left(u_{s + 1} + \frac{12 \rho^{4} u_{\infty} \beta_s}{\eta} + (C_1 \cdot b(n,s,k) \frac{\rho^2 r(\eta)}{\sqrt{n}}) \| g \|_2^{2}\right).  
\end{align*}  
\end{lemma}

\begin{proof}
For the minimization problem (\ref{poly-krr}), the solution is 
\[
\hat{f}=\sum_{i=1}^{n}\hat{c}_iK(x, x_i)+\sum_{k = 1}^{s} \hat{\lambda}_{k}p_k(x),
\]
with kernel coefficients $\hat{c}=(\hat{c}_1,\ldots,\hat{c}_n)=(K+\eta I)^{-1}(Y-P\hat{\lambda})$ and polynomial coefficients $\hat{\lambda}=(\hat{\lambda}_1,\ldots,\hat{\lambda}_s)=(P^T(K+\eta I)^{-1}P)^{-1}P^T(K+\eta I)^{-1}Y$. Then $$\E \| \hat{f} - f \|_{2}^{2}  = \E \|\sum_{i=1}^{n}\hat{c}_iK(x, x_i) - g(x)+\sum_{k=1}^{s} p_k(\hat{\lambda}_k-\lambda^*_k)\|_{2}^{2},$$ where $g(x) = f(x) -  \sum_{k = 1}^{s} p_k(x) \lambda^{*}_{k}$. Write the Taylor expansion of $f(x)$ at $0$ as $f(x) = \sum_{k = 1}^{\infty} L_{k}(f(0))p_{k}(x)$, we have $\lambda^*_k=L_{k}(f(0))$ and $g(x) = \sum_{k = s + 1}^{\infty} L_{k}(f(0)) p_{k}(x)$, where $L_{k}$ is the differential operator. Then we have $$\E \| \hat{f} - f \|_{2}^{2}\le \E \| \sum_{i=1}^{n}\hat{c}_iK(x, x_i) - g(x) \|_{2}^{2}+\E \| \sum_{k=1}^{s} p_k(\hat{\lambda}_k-\lambda^*_k) \|_{2}^{2}.$$

The  term $\sum_{i=1}^{n}\hat{c}_iK(x, x_i)$ can be treated as a solution to the KRR problem with data $\{(x_i,z_i=y_i-\sum_{j=1}^{s}\lambda^*_jp_j(x_i))\}_{i=1}^{n}$ from the model $Z=Y-\sum_{k = 1}^{s} p_k(X) \lambda^{*}_{k}=f(X)+\varepsilon-\sum_{k = 1}^{s} p_k(X) \lambda^{*}_{k}=g(X)+\varepsilon$. By Assumption \ref{conv_assump_1}, we have $\E[(z_i - g(x_i))^2 \mid x_i ]= \E[(y_i - f(x_i))^2 \mid x_i ]\le \delta^2$ for some $\delta > 0$. Together with Assumption \ref{conv_assump_2}, by  Lemma 7 in ~\cite{JMLR:v16:zhang15d}, we have the KRR solution $\sum_{i=1}^{n}\hat{c}_iK(x, x_i)$ satisfies 
\begin{align*}
\E \| \sum_{i=1}^{n}\hat{c}_iK(x, x_i) - g(x) \|_{2}^{2} 
  & \le 
     12 \eta \| g\|_{\cal{H}}^{2} + \frac{12 \delta^{2}}{n} r(\eta) 
  \\ & 
    + \hspace{-1mm}\left(\frac{2 \delta^2}{\eta} + \hspace{-1mm}4 \| g \|_{\cal{H}}^{2}\right) \hspace{-1mm}\cdot\hspace{-1mm} \left(u_{s + 1} + \hspace{-1mm}\frac{12 \rho^{4} u_{\infty} \beta_s}{\eta} + \hspace{-1mm}(C_1 \hspace{-1mm}\cdot b(n,s,k) \frac{\rho^2 r(\eta)}{\sqrt{n}}) \| g \|_2^{2}\right)\hspace{-1mm}.
\end{align*}

From the paper \cite{BAYONA20192337}, we know that $|\hat{\lambda}_k-\lambda^*_k|=o(h^{l+1})$. So 
$$\E \| \sum_{k=1}^{s} p_k(\hat{\lambda}_k-\lambda^*_k) \|_{2}^{2}\le s\cdot E\|p_k\|_2^2\cdot o(h^{2(l+1)}).$$

So the desired result of Lemma \ref{conv_lemma} holds.
\end{proof}

\paragraph*{Remark}
The term $u_{s+1}$ and $\beta_s$ are decreasing functions of $d$. The term $b(n,s,k)$ is increasing in $s$. By carefully picking $s$, the first two terms in the theorem are dominant. Then the theorem indicates that the mean squared error of our estimator, $\E \| \hat{f} - f \|_{2}^{2}$, is upper bounded by a function of $\eta$, $n$ and $s$. In fact, it indicates that $\E \| \hat{f} - f \|_{2}^{2}=O((\eta +r(\eta)/n)\cdot \| g \|_{\cal{H}}^{2})$. This is a common bias-variance trade-off inequality in non-parametric regression problems, where the first term $(\eta)$ is increasing in $\eta$ and the second term $r(\eta)/n$ is deceasing in $\eta$. Note that the usual KRR estimator has the bound $O((\eta +r(\eta)/n)\cdot \| f \|_{\cal{H}}^{2})$, and $r(\eta)$ can be bounded by $1/\eta$.
Thus Theorem \ref{conv_theorem} follows.  



\section{Experiments}

We compare PU-KRR and PU-KRR-POLY with global KRR and other stitched kernel regression methods -- specifically, KNN SVM~\cite{Hable_Robert}, local SVM~\cite{Ingo_Steinwart}, and local KRR~\cite{hanlocal} -- as well as neural nets and random forests.  In some cases, for completion, we also ran PU-POLY which uses PU-stitched regression with a degree-2 polynomial model used locally.  
We measure on a 2D synthetic dataset, a 3D data set from the solution of a PDE on the sphere~\cite{SHANKAR2018170}, and several real-world datasets.  The results from some generic and higher-dimensional datasets from the UCI repo~\cite{Dua:2019} are deferred to the Appendix. 
The error metrics we employ to compare experimental results include: $\rmse$ and relative error, which is $\text{rel-err}_i = \frac{\|y_i - \hat f(x_i)\|}{\|y_i\|}$. 
For all experiments, we do train/validation/test split, and for all the kernel related methods we tune the bandwidth $\sigma$ and ridge parameter $\eta$ by grid search. We select $\eta$ from $\{1\text{e-}1,1\text{e-}2,1\text{e-}3,1\text{e-}4,1\text{e-}5\}$ and $\sigma$ from $\{ 0.25 * b, 0.5 * b, b, 2b, 5b \}$ where $b$ is the average pairwise distance in each local/global model(s). For the tree method XGBoost~\cite{DBLP:journals/corr/ChenG16}, we also tune hyper-parameters e.g: number of subtrees, learning rate, etc. by grid search. Finally for the neural network, we have tried 2-5 hidden layers with different number of neurons varied in $\{8,16,32,64,128,256\}$, and report the best results found.  All the experiment results are reported based on a held-out test set. 

\subsection{2D Tests} 
\label{sec:2DTests}
To demonstrate that our method can adapt to different scales of response values in different regions, we design a 2D dataset where the underlying function values $y$ are generated by $x_1, x_2 \in [-6,30] \times [-6,30]$, $z_1 = \frac{1}{1 + \exp(-x_1)} \cdot (1 + \frac{9}{1 + \exp(12 - x_1)}) \cdot (1 + \frac{10}{1 + \exp(24 - x_1)})$,  $z_2 = \sin(x_2) + \cos(x_1)$, $y = z_1 \cdot z_2$.  As seen in the "True y" plot in Figure \ref{fig:2d}, the function values are much closer to zero when $x_1$ is small $x_1 \in [-6, 6]$, moderate variance in $x_1 \in [6,18]$, and large variation in $x_1 \in [18,30]$.  
We randomly selected 20000 points from $[-6,30] \times [-6,30]$ as training dataset, then choosing test set based a fine grid with a data point $(x_1, x_2) \in [-6,30] \times [-6,30]$ every $0.2$. We tune the parameters based on the performance in training set and report the test set errors in Figure \ref{fig:2d} and Table \ref{tbl:2d}. Note how by design the response value $y$ becomes larger on the right where additive error tends to be larger, and smaller on the left where relative error tends to be larger.   

PU-KRR-POLY has the least RMSE (by almost an order of magnitude) and the best or near-best max and mean relative error.  The other KRR-based models all achieve similar RMSE (between $0.25$ and $0.35$) and also similar relative error.  However, Figure \ref{fig:2d} shows several artifacts. For instance, global KRR exhibits a banding effect in the error, which can be seen at a much smaller scale in the other KRR models; this is an example of a boundary or far-field effect~\cite{fornberg2002} that adding the polynomial terms is designed to remove.  Also, a visible discontinuity is apparent in Local SVM.  
The XGBoost and Neural Network models 
do not at all perform well on this data set as apparent visually and quantitatively.  Especially for the neural network, since the input features is with only 2 dimension, we use a 2-layer MLP (multiple layer perceptron) with sigmoid activation function (it performed worse with 3 or 4 layers); the network structure is tuned by AX package using Monte Carlo Bayesian Optimization \cite{balandat2020botorch}.

\begin{figure}[hbtp!]
    \centering
    \includegraphics[width=0.9\linewidth]{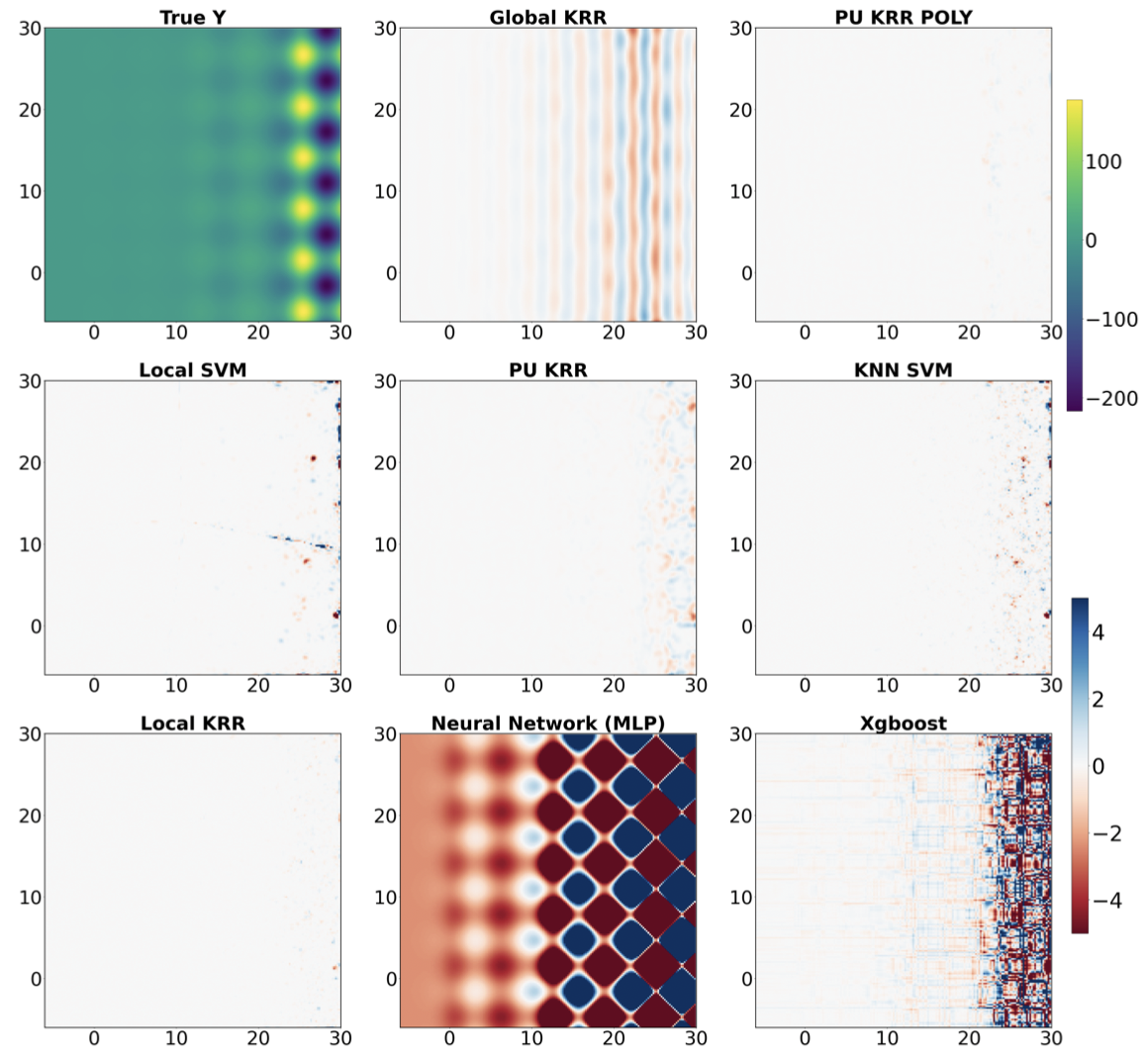}
    \caption{Error plots for 2D simulation result on data in \textbf{\sffamily True y} plot; summarized in Table \ref{tbl:2d}.}
    \label{fig:2d}
\end{figure}

\begin{table}[hbtp!]
\centering
\begin{tabular}{rccc}
\hline
     &  \rmse & max relative error & mean relative error \\ \hline
PU-KRR-POLY &  0.041  & 50.392  & 0.021  \\ 
        PU-KRR & 0.249 & 49.928 & 0.026  \\ 
    PU-POLY & 0.239 & 46.052 & 0.037 \\
KNN SVM & 0.337  & 47.284 & 0.027 \\
Local SVM & 0.321  & 44.421 & 0.026 \\ 
Local \krr &  0.089 & 53.420  & 0.023 \\ 
Global \krr & 0.354  & 327.301 & 0.531 \\
XGBoost & 2.729  & 1170.34  & 3.359 \\
Neural Network & 7.525 & 3042.82 & 14.185 \\
\hline
\end{tabular}
   \caption{RMSE $\&$ Relative errors comparison for 2D simulation}
   \label{tbl:2d}
\end{table}

\paragraph*{Ablation Study}
To verify the choice of parameters and design decisions, we perform an ablation study.  We focus on the simulated 2D dataset discussed in Table \ref{tbl:2d} and Figure \ref{fig:2d}.  

First, we try an additional regression model, and show the results in Table \ref{tbl:2d}.  PU-POLY uses PU-stitched regression, but with only a degree-2 polynomial model as the local model.  It performs similar, but a bit worse than PU-KRR and worse in RMSE and mean-relative-error than PU-KRR-POLY.

Second, recall we selected the centers of the regions arbitrarily, so long as they were not yet covered.  We implement this by just scanning all data points in the order they are stored until we find one not yet covered or all are covered.  To assess the stability of this process, we randomized the order of the points 3 times, and re-reran PU-KRR-POLY on the new sets of regions.  The average RMSE was $0.045$ with standard deviation $0.005$.  So about the same as the intial run.  The average (std.dev) for max relative error $41.36(15.82)$ and mean relative error $0.025(0.0002)$; also similar to the results for the input order PU-KRR-POLY.  Note the high standard deviation of $15.82$, which indicates while the maximum relative error is a useful goal, it is not a very stable measure compared to RMSE and mean relative error.

Third, we experimented with the value $h$ describing the number of points to include in a region.  We compare to a default of $h=100$, and tried $h = \{50, 75, 125, 150\}$.  We see in Table \ref{tbl:h} that there is not much change, but that, relative to $h=100$, as $h$ increases the RMSE increases, and as $h$ decreases mean relative error also increases. Hence $h=100$ appears a good choice, and varying this within $50\%$ should not change things too much, although increasing it can cause error to increase since the local models may not fit the data as well anymore.

\begin{table}[hbtp!]
\centering
\begin{tabular}{rcccc}
\hline
     & $h$ & \rmse & max relative error & mean relative error \\ \hline
PU-KRR-POLY & 100 &  0.041  & 50.392  & 0.021  \\ \hline

PU-KRR-POLY & 50 & 0.038 & 40.371 & 0.037 \\
PU-KRR-POLY & 75 & 0.043 & 94.477 & 0.029 \\
PU-KRR-POLY & 125 & 0.067 & 95.709 & 0.055 \\
PU-KRR-POLY & 150 & 0.110 & 65.705 & 0.064 \\
\hline
\end{tabular}
   \caption{Ablation study for PU-KRR-POLY on $h$, the number of points in a local region.}
   \label{tbl:h}
\end{table}

Fourth, we check how the algorithm varies with the choice of kernel used to provide weights to guide how local regions are stitched together.  In our experiments we use the Wendland $\phi_{2,1}$ kernel, which provides $C^2$-continuity.  For the algorithm to build local models, we require kernels with bounded support.  In Table \ref{tbl:Wendland} we also consider the Gaussian (truncated), which does not ensure continuity, and the Wendland $\phi_{2,0}$ kernel which only provides $C^0$-continuity.  We see in Table \ref{tbl:Wendland} that this change results in a very small effect in the error measures, so it is worth using the Wendland $\phi_{2,1}$ kernel which has stronger gaurantees.

\begin{table}[hbtp!]
\centering
\begin{tabular}{rcccc}
\hline
     & stitching kernel & \rmse & max relative error & mean relative error \\ \hline
PU-KRR-POLY & Wendland $\phi_{2,1}$ &  0.041  & 50.392  & 0.021  \\ \hline
PU-KRR-POLY & Wendland $\phi_{2,0}$ &  0.048 & 30.590 & 0.024 \\
PU-KRR-POLY & trunc-Gaussian & 0.041 & 51.691 & 0.024 \\
\hline
\end{tabular}
   \caption{Ablation study for stitching kernel in PU-KRR-POLY.  Default is Wendland$_{3,1}$.}
   \label{tbl:Wendland}
\end{table}

\subsection{Simulation data on the sphere $\mathbb{S}^2$}
In this experiment, we generate the training data by numerically solving the spherical advection equation $\frac{\partial q}{\partial t} + {\bf u}\cdot\nabla_{\mathbb{S}^2} = 0$ in 3D Cartesian coordinates using a fourth-order accurate semi-Lagrangian local RBF method~\cite{SHANKAR2018170}, where $q(x,t)$ is a scalar-valued function; here, $\nabla_{\mathbb{S}^2} = (I - xx^T)\nabla$ is the \emph{surface} gradient. The initial condition to this problem is a pair of $C^1(\mathbb{S}^2)$ cosine bells given by $q(x,0) = 0.1 + 0.9(q_1(x,0) + q_2(x,0))$, where for $j=1,2$
\begin{align*}
q_j(x,0) &=
\begin{cases}
\frac{1}{2}\left(1 + \cos\left(2\pi \cos^{-1}\left(x \cdot p_j\right) \right) \right) & \text{if}\ \cos^{-1}\left(x \cdot p_j\right) < \frac{1}{2}, \\
0 & \text{otherwise}.
\end{cases}
\end{align*}
The components of the velocity field ${\bf u} = (u,v)$ for this test are given in Eqns. 16 and 17 in~\cite{SHANKAR2018170}. This flow field is designed to deform the initial condition and reverse it back at time $t=5$. The RBF-FD solution $q$ was computed at ~92k equal-area icosahedral Cartesian points on the sphere. To illustrate the ability of our method to handle concentrated features with values close to 0, we shift the solution to be $\hat{q} = \max(q) - q$. The resulting function values are shown in Figure \ref{fig:3d_ball}. We then sampled from the original set based on Bernoulli trials with the probability proportional to the inverse of distance to the centers of two bump regions; the details of this sampling strategy are described in Appendix ~\ref{appendix:sampling}. 

We plot the average of five trials of \rmse, mean relative error, max relative error on a fixed number of random sampled test observations (20,000), versus different training sample sizes $\{3065, 5377, 11504, 21545\}$ for different methods in Figure~\ref{fig:error_sample}. For the global KRR, we were only able to use training sizes less than 15000 due to memory limitations. At every training size, PU-KRR-POLY provides the best result, with PU-KRR typically second or near-second best, especially on larger training sets.  KNN SVM and local SVM are similar to PU-KRR, but plateau in relative error for larger training set size, likely because of lack of continuity.

\begin{figure}[hbtp!]
    \centering
    \includegraphics[width=0.41\linewidth]{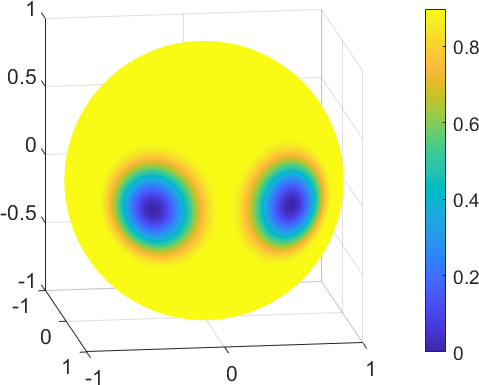}
    \caption{\label{fig:3d_ball}Function values on $\mathbb{S}^2$} 
\end{figure}
\begin{figure}[hbtp!]
    \centering
    \includegraphics[width=1.1\linewidth]{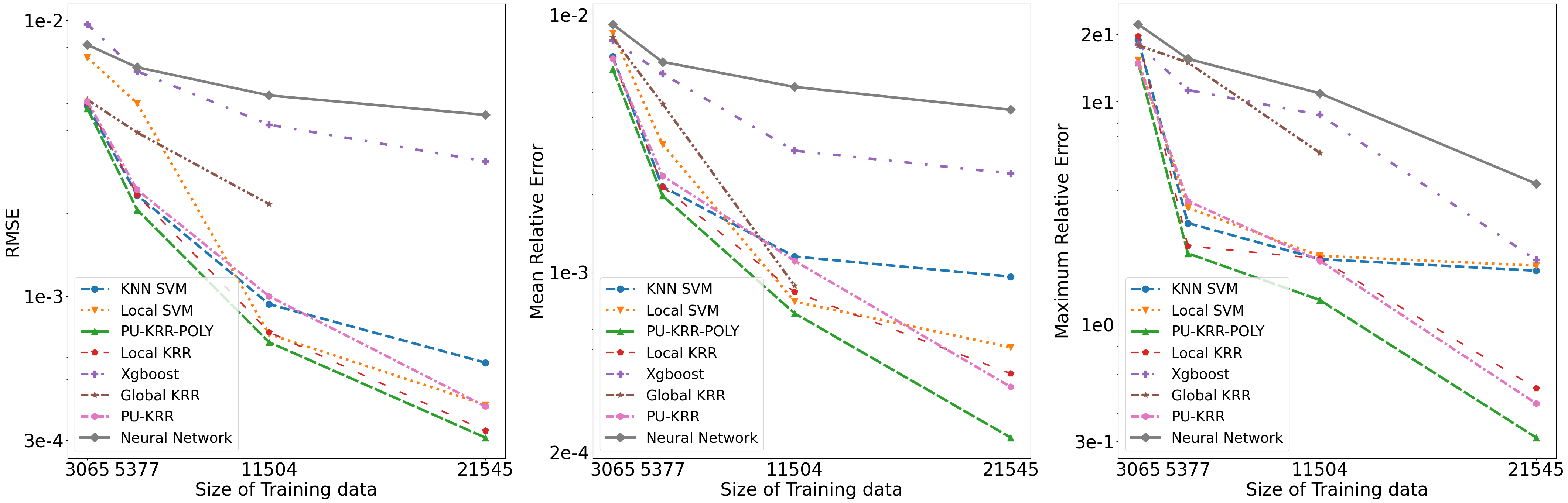}
    \caption{Error vs. Training size on PDE-on-sphere data}
    \label{fig:error_sample}
\end{figure}

\begin{figure}[hbtp!]
    \centering
 \centering
    \centering
    \includegraphics[width=0.41\linewidth]{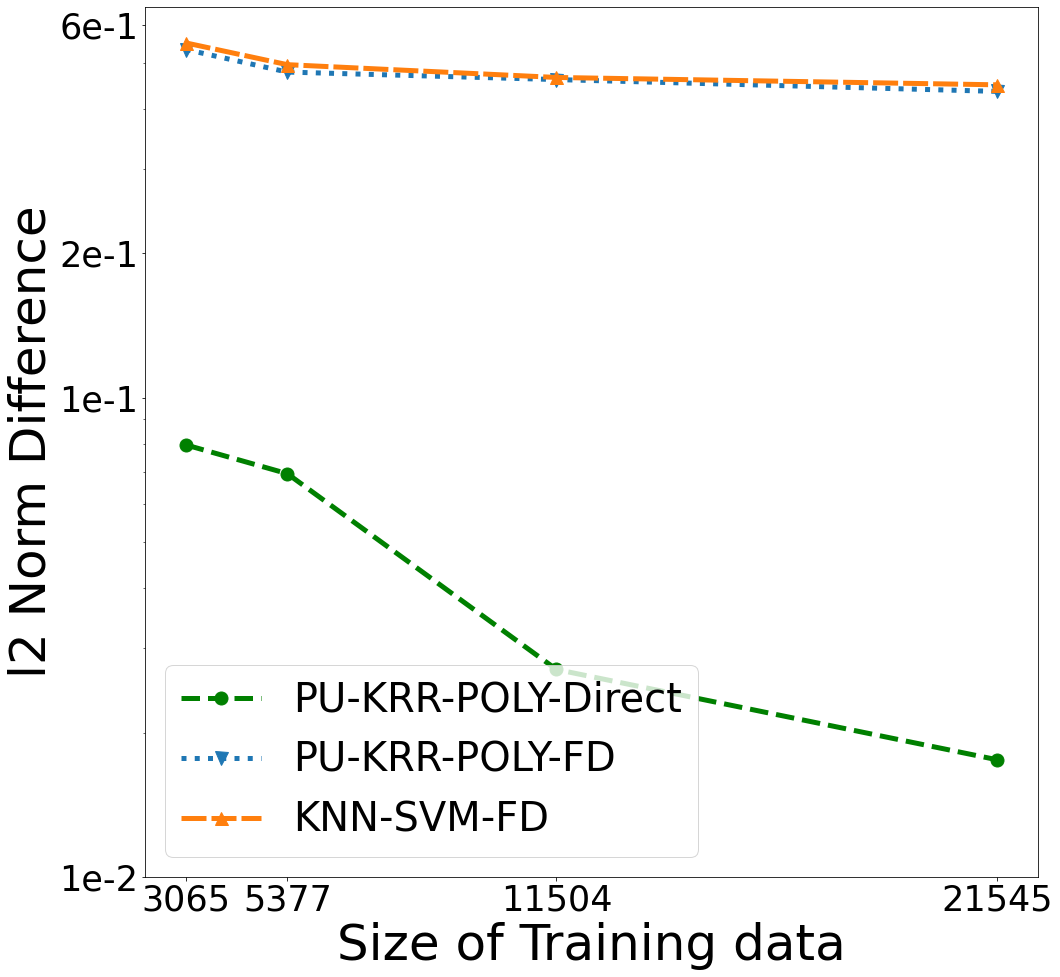}
    \caption{Average of $\ell_2$ norm difference in Gradient}
    \label{fig:3D_ball_and_grad}
\end{figure}

As we discuss in Sec \ref{sec:derivative}, our proposed PU-Stitched scheme can directly provide derivatives. We now test the ability of our scheme to compute surface gradients, despite the use of Cartesian coordinates. As our ground truth, we use a fourth-order accurate RBF-FD method from~\cite{SNKJCP2018} to compute the surface gradient on the ~92k icosahedral points. We then compute the surface gradient using PU-KRR-POLY by applying the surface gradient operator in Cartesian coordinates in place of the regular derivative operator in \eqref{eq:derivatives}. We show the average of the difference in pointwise $\ell_2$ norms between the RBF-FD gradient and the PU-KRR-POLY gradient in Figure \ref{fig:3D_ball_and_grad}. For comparison, we also computed surface gradients using forward differencing with the nearest Cartesian neighbor for both KNN-SVM and PU-KRR-POLY, and compared them to the RBF-FD surface gradient. The PU-KRR-POLY direct method is orders of magnitudes more accurate.


\subsection{Spatial Ozone Data} 
We compare PU-KRR-POLY with other methods on the Ozone levels from \cite{DI2019104909}, which are recorded on 1km by 1km grids in a lat-long bounding box of $[36.3,42.6] \times [-114.6,-108.4]$ over a mountainous part of the U.S.; they are plotted on the left in Figure ~\ref{fig:ozone_1}.   The prediction error $y_i - \hat{f}(x_i)$ for PU-KRR-POLY is plotted to its right. 
In all error measures, PU-KRR-POLY performs the best, with PU-KRR typically second, followed by the other stitched KRR models.  Relative error is less pertinent here as values are not close to $0$.  

The Appendix provides results on other generic UCI data sets designated for regression tasks, including ones in higher dimensions.  PU-KRR-POLY is typically, but not always the best -- unlike the data sets explored here, these generic data sets do not have a large benefit from local models in local regions that have different density of response variance.  

\begin{figure}
\centering
\begin{minipage}{0.43\linewidth}
    \includegraphics[width=\linewidth]{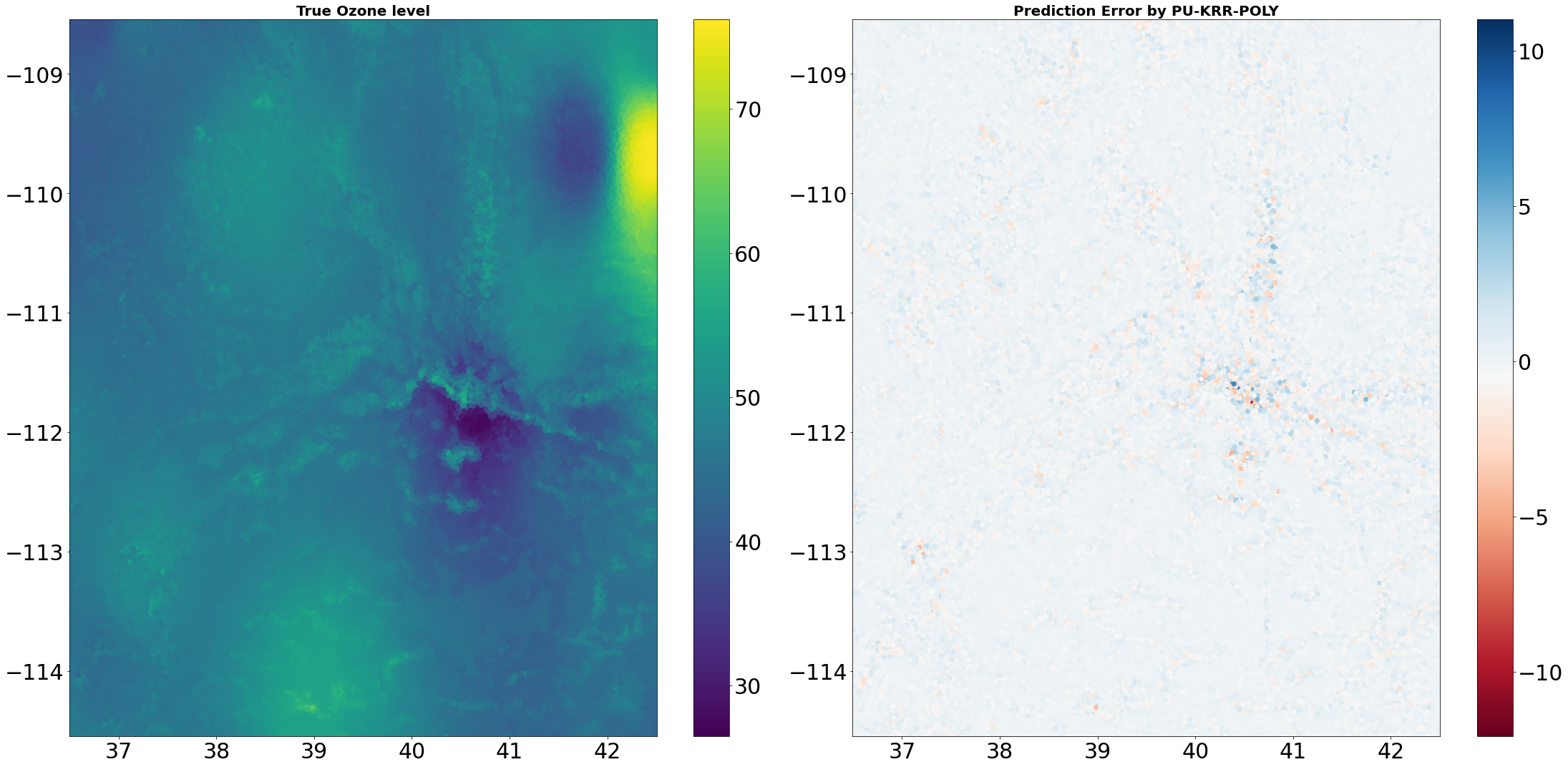}
\end{minipage}\hfill
\begin{minipage}{0.57\linewidth}
 \centering
\begin{tabular}{rccc}
  \hline
     &  \rmse & max rel.er & mean rel.er \\ \hline
PU-KRR & 0.554 & 0.303 & 0.007  \\ 
PU-KRR-POLY & 0.542  & 0.287  & 0.007 \\ 
XGBoost & 1.223 & 0.354 &  0.017 \\
Neural Network & 1.756 & 0.744  & 0.028 \\
KNN SVM &  0.611 & 0.291 &  0.008 \\
Local SVM & 0.729 & 0.332 & 0.010 \\ 
Local \krr & 0.548 & 0.312 & 0.007 \\ 
\hline
 \end{tabular}
\end{minipage}
\caption{Spatial Ozone Data, and error in recovery by methods}
\label{fig:ozone_1}
\end{figure}

\subsection{Combustion data}

To further illustrate motivations of PU-KRR-POLY, we consider a setting with data generated from an unknown PDE; that is the data is experimentally observed, and one knows it should be governed by an underlying system, but that system is not known.  We experiment in this setting on combustion data~\cite{GRI-Mech} generated from a PDE of the form  $S(\delta) + \frac{\chi(Z)}{2} + \pdv[2]{\delta}{Z}= 0$, with $S(\delta)$ is complicated and assumed not to be known (here the values of $S(\delta)$ are generated with a dimension-compressed solver~\cite{sutherland2009combustion}). In particular, for the experimental set-up, our training variables are $\delta$ and response variables are $S(\delta)$. 
The main goal is try to give a continuous, differentiable, and accurate predictions to $S(\delta)$ based on $\delta$.  In such cases, most neural ODE/PDE solvers~\cite{neural_ode_chen, neural_pde_solver_Johannes} need an explicit form of PDE equations and thus cannot be directly applied here.  We show the comparison results in Table ~\ref{combustion_table}, including against SIREN~\cite{SIREN}.  
Again PU-KRR-POLY outperforms all other methods by 1-2 orders of magnitude in RMSE, max relative error, and mean relative error.   

\begin{table}[h]
\centering 
\begin{tabular}{rccc}
\hline
     &  \rmse & max relative error & mean relative error \\ \hline
PU-KRR  &  3.335e-4 &  1.033e-2 & 1.807e-4 \\
PU-KRR-POLY & 9.303e-6  & 6.827e-3 & 1.041e-5 \\
KNN SVM & 5.561e-2 & 3.223e-1  & 6.372e-3  \\
Local SVM &  3.948e-2 & 9.844e-1 & 4.566e-2   \\ 
Local \krr & 7.639e-4 & 1.279e-2 & 1.994e-4 \\
XGBoost & 1.931e-3 & 1.086e0  &  4.569e-3\\
Neural Network & 5.516e-3  & 4.586e0 & 2.450e-2 \\
SIREN & 6.470e-2 & 4.336e0 & 3.571e-2 \\
\hline
    \end{tabular}
    \caption{Combustion Data Results}
    \label{combustion_table}
\end{table}

\vspace{-1mm}

\section{Discussion}

We describe new locally-adaptable regression model which can ensure $C^t$-continuity for any finite $t$, and for which derivatives can be automatically computed.  It leverages a partition-of-unity stitching of local models.   
We also propose and analyze a regression model that orthogonally mixes polynomial and kernel ridge terms with improved statistical convergence and empirical performance.  
We find the PU-KRR-POLY regression model is efficient and outperforms other advanced regression models, especially on data which benefit from locally-adaptable continuous models.

\paragraph*{Efficiency and Runtime}
Our method is reliant on nearest neighbor searching to both determine the points within a local neighborhood and also determine which local models cover a query point.  
Fortuitously, similarity search on vector data, has somewhat recently become extremely efficient~\cite{andoni2015practical,li2019approximate,ram2019revisiting,Johnson2021BillionScaleSS}.  

Let $T(n,d)$ be the time it takes to perform a nearest neighbor query on $n$ data points in $\R^d$.  While the best theoretical bounds are either mediocre or nuanced, in practice it is quite efficient.  
When each model region is set to have a constant number of points, so $h = O(1)$, then we can determine and build all local models in $O(n \cdot T(n,d))$ time.

To query the model at a single location $q \in \R^d$, we need to perform a reverse metric range query -- find all region balls $B_j$ which contain $q$.  To bound the complexity of this we introduce two common data-dependent parameters: the spread $\Lambda$ (measuring precision) and the doubling dimension $\ddim$ (measuring intrinsic dimensionality).  Specifically, define $\Lambda$ as the ratio between the largest and smallest radius of the model-containing balls.  
We group balls into levels depending on their radius, so within a level all radii are within a factor of $2$; there are at most $\log \Lambda$ such levels.  
For a point set $M$, the doubling dimension $\ddim$ is defined as the maximum of the log of a quantity over all balls $B$ of radius $r$; the quantity is the minimum number of balls or radius $r/2$ needed to cover $M \cap B$.  Note that $\ddim \leq d$ and typically much smaller for real high-dimensional data.  
Now within each level at most $O(2^\ddim)$ balls can contain a query $q$ since in the creation process, each new ball must contain a point not in any previous balls.  
So we query each level, and retrieve at most $O(2^\ddim)$ balls with centers within the largest radius of that level.  Evaluating a local model is $O(1)$ time for $h = O(1)$.  Hence, a query takes $O((2^\ddim + T(n,d))\log \Lambda)$ time.  In our settings, $2^\ddim$ and $\log \Lambda$ are small constants (typically not more than $10$), and this is quite efficient.  

The bottom line is that the model building step runs in time roughly linearly in $n$, and the model evaluation time runs in time significantly sublinear in $n$.  For the experiments we ran, PU-KRR-POLY had similar (or much faster) runtime to all other competing methods, and so scalability was not a concern relative to any other standard regression method.

Moreover, when a gradient is needed in $\mathbb{R}^d$, one simply needs to differentiate the PU-stitched model analytically, which in turn only requires derivatives of the Wendland functions, the kernel, and the polynomial terms. In contrast, typical approaches compute a discrete gradient with $d+\delta u_j$ for a small $\delta$ at each of $d$ orthogonal basis vectors $u_j$.  By not requiring these extra evaluations, our PU-stitched methods can save a factor of $d$, and as observed, improve accuracy.

\bibliographystyle{nipsbib}
\bibliography{nips}

\clearpage
\appendix

\section{Real Data Experiments from UCI Repository}\label{uci_data}

Here we provide results on other generic UCI data sets designated for regression tasks, including ones in higher dimensions.  PU-KRR-POLY is typically, but not always the best -- unlike the data sets explored here, these generic data sets do not have a large benefit from local models in local regions that has difference density of response variance.

Skillcraft data has 3395 observations and 19 features in total. See in \url{http://archive.ics.uci.edu/ml/datasets/skillcraft1+master+table+dataset}.
\begin{table}[h]
\centering
\begin{tabular}{rccc}
\hline
     &  \rmse & max relative error & mean relative error \\ \hline
PU \krr poly &  0.246 & 33.179 & 1.290 \\
PU \krr & 0.283 & 47.603 & 1.693 \\
Global \krr & 0.252 & 44.866 & 1.304 \\
        XGBoost & 0.257 & 57.889 & 1.492 \\
Neural Network & 0.249 & 37.656 & 1.293 \\
KNN SVM & 0.248 & 65.457 & 1.295 \\
Local SVM & 0.256 & 34.865 & 1.315 \\ 
Local \krr & 0.249 & 40.879 & 1.391 \\
\hline
\end{tabular}
   \caption{rmse $\&$ relative errors: skillcraft data}
   \label{tbl:2}
\end{table}

Airfoil data has 1503 observations and 6 features. See in \url{https://archive.ics.uci.edu/ml/datasets/airfoil+self-noise}.
\begin{table}[h]
\centering
\begin{tabular}{rccc}
\hline
     &  \rmse & max relative error & mean relative error \\ \hline
PU \krr poly & 1.702 & 13.114 & 0.601 \\
PU \krr  & 1.749 & 15.856 & 0.615 \\ 
Global \krr & 2.677 & 14.619 & 0.833 \\
XGBoost & 1.588 & 15.023 & 0.554\\ 
Neural Network & 1.826 & 15.926  & 0.791 \\ 
KNN SVM & 1.687 & 10.717 & 0.665 \\
Local SVM & 2.357 & 13.817 & 0.956  \\ 
Local \krr & 1.499 & 15.856 & 0.615 \\
\hline
\end{tabular}
   \caption{rmse $\&$ relative errors: air-foil data}
   \label{tbl:3}
\end{table}

Kin40k has 40000 observations and 8 features in total. See in \url{https://github.com/treforevans/uci_datasets}.
\begin{table}[h]
\centering
\begin{tabular}{rccc}
\hline
     &  \rmse & max relative error & mean relative error \\ \hline
PU \krr poly & 0.143 & 384.125 & 0.556  \\
PU \krr & 0.124 & 306.742 &  0.364 \\
Global \krr & 0.146  & 418.279 & 0.585 \\
XGBoost & 0.316 & 1134.082 & 1.445 \\
Neural Network & 0.139 & 227.806 & 0.494\\
KNN SVM & 0.141 & 218.083 & 0.459 \\
Local SVM & 0.145  & 384.432 &  0.557 \\ 
Local \krr & 0.124 & 106.743  & 0.364 \\
\hline
\end{tabular}
   \caption{rmse $\&$ relative errors: kin40k data}
   \label{tbl:4}
\end{table}

\section{Sampling strategy in 3D experiments}\label{appendix:sampling}

Let's denote the centers of bumpers in Figure ~\ref{fig:3d_ball} as $c_1$ and $c_2$. For each point $x$, we run two independent Bernoulli trials to determine it would be selected as a training data point or not. In the $1_{st}$ Bernoulli trial, the probability it would be selected as training data is $1 - \frac{\|x - c_1\|}{r_1}$, where $r_1$ is distance between $c_1$ and its furthest point. In the $2_{nd}$ Bernoulli trial, the probability it would be selected as training data is $1 - \frac{\|x - c_2\|}{r_2}$, where $r_2$ is distance between $c_2$ and its furthest point. A point is selected only if it was selected in either $1_{st}$ or $2_{nd}$ trial, so equivalent the probability it would be selected as training point with probability equal to $1 - (\frac{\|x - c_2\|}{r_2}) \cdot (\frac{\|x - c_1\|}{r_1})$.

\end{document}